\documentclass{article}

\usepackage{microtype}
\usepackage{graphicx}
\usepackage{subfigure}
\usepackage{booktabs} 
\usepackage{hyperref}

\usepackage[accepted]{icml2019}

\usepackage{hyperref}       \usepackage{url}            \usepackage{nicefrac}       \usepackage{amsthm}
\usepackage{amsmath}
\usepackage{amssymb}
\usepackage{tikz}

\usepackage[mathscr]{euscript}
\usepackage{algorithm}

\usepackage[normalem]{ulem}

\newcommand{\fgwdistance}{{FGW}_{q,\alpha}(\mu,\nu)}
\newcommand{\couplingset}{\Pi}

\theoremstyle{definition}

\theoremstyle{theorem}
\newtheorem{theorem}{Theorem}[section]
\theoremstyle{remark}

\newtheorem{lemma}[theorem]{Lemma}

\begin{document}

\twocolumn[
\icmltitle{Optimal Transport for structured data with application on graphs}
\icmlsetsymbol{equal}{*}

\begin{icmlauthorlist}
\icmlauthor{Titouan Vayer}{to}
\icmlauthor{Laetitia Chapel}{to}
\icmlauthor{R\'emi Flamary}{goo}
\icmlauthor{Romain Tavenard}{ed}
\icmlauthor{Nicolas Courty}{to}
\end{icmlauthorlist}

\icmlaffiliation{to}{Univ. Bretagne-Sud, CNRS, IRISA, F-56000 Vannes}
\icmlaffiliation{goo}{Univ. C\^ote d'Azur, CNRS, OCA Lagrange, F-06000 Nice}
\icmlaffiliation{ed}{Univ. Rennes, CNRS, LETG,  F-35000 Rennes}

\icmlcorrespondingauthor{Titouan Vayer}{titouan.vayer@irisa.fr}

\vskip 0.3in

]

\printAffiliationsAndNotice{}

\begin{abstract}
This work considers the problem of computing distances between structured
objects such as undirected graphs, seen as probability distributions in a
specific metric space. We consider {\bf a new transportation distance} ({\em
i.e.} that minimizes a total cost of transporting probability masses) that unveils
the geometric nature of the structured objects space. Unlike Wasserstein or
Gromov-Wasserstein metrics that focus solely and respectively on features (by
considering a metric in the feature space) or structure (by seeing structure as
a metric space), our new distance exploits jointly both information, and is
consequently called Fused Gromov-Wasserstein  (FGW). After discussing its
properties and computational aspects, we show results on a graph classification
task, where our method outperforms both graph kernels and
deep graph convolutional networks.  Exploiting further on the metric properties
of FGW, interesting geometric objects such as Fr{\'e}chet means or barycenters
of graphs are illustrated and discussed in a clustering context.

\end{abstract}

\section{Introduction}

There is a longstanding line of research on learning from structured data, {\em i.e.} objects that are a combination of a feature and structural information (see for example \cite{Bakir:2007:PSD:1296180,relationnalreasoning}). As immediate instances, graph data are usually ensembles of nodes with attributes (typically  $\mathbb{R}^{d}$ vectors) linked by some specific relation. Notable examples are found in chemical compounds or molecules modeling~\cite{DBLP:journals/corr/KriegeGW16}, brain connectivity~\cite{ktena2017distance}, or social networks~\cite{Yanardag15}. This generic family of objects also encompasses time series \cite{pmlr-v70-cuturi17a}, trees~\cite{day1985optimal} or even images~\cite{bachgraphkernel}.

Being able to leverage on both feature and structural information in a learning task is a tedious task, that requires the association in some ways of those two pieces of information in order to capture the similarity between the 
structured data. Several kernels have been designed to perform this task~\cite{Shervashidze:2011:WGK:1953048.2078187,wlkernel}. As a good representative of those methods, 
the Weisfeiler-Lehman kernel~\cite{wlkernel} captures in each node a notion of vicinity by aggregating, in the sense of the topology of the graph, the surrounding features. Recent advances in graph convolutional networks~\cite{DBLP:journals/corr/BronsteinBLSV16, Kipf2016SemiSupervisedCW,NIPS2016_6081} allows learning end-to-end the best combination of features by relying on parametric convolutions on the graph, {\em i.e.} learnable linear combinations of features. In the end, and in order to compare two graphs that might have different number of nodes and connections, those two categories of methods build a new representation for every graph that shares the same {space}, and that is amenable to classification.

\paragraph{A transportation distance between structured data.}
Contrasting with those previous methods, we suggest in this paper to see graphs as probability distributions, embedded in a specific metric space. We propose to define a specific notion of distance between those probability distributions, that can be used in most of the classical machine learning approaches. Beyond its mathematical properties, disposing of a distance between structured data, provided it is meaningful, is desirable in many ways: {\em i)} it can then be plugged into distance-based machine learning algorithms such as $k$-nn or t-SNE {\em ii)} its quality is not dependent on the learning set size, and {\em iii)} it allows considering interesting quantities such as geodesic interpolation or barycenters. To the best of our knowledge, this is one of the first attempts to define such a distance on structured data.

Yet, defining this distance is not a trivial task. While features can always be compared using a standard metric, such as $\ell_2$, comparing structures requires a notion of similarity which can be found \textit{via} the notion of {\em isometry}, since the graph nodes are not ordered (we define later on which cases two graphs are considered identical). We use the notion of transportation distance to compare two graphs represented as probability distributions. Optimal transport (OT) have inspired a number of recent breakthroughs in machine learning (\emph{e.g.} \cite{huang2016,courty2017optimal,arjovsky17a}) because of its capacity to compare empirical distributions, and also the recent advances in solving the underlying problem~\cite{peyre2018computational}. Yet, the natural formulation of OT cannot leverage the structural information of objects since it only relies on a cost function that compares their feature representations.

However, some modifications over OT formulation have been proposed in order to compare structural information of objects. Following the pioneering work by M\'emoli \cite{springerlink:10.1007/s10208-011-9093-5}, Peyr\'e {\em et al.} \cite{peyre2016gromov} propose a way of comparing two distance matrices that can be seen as representations of some objects' structures. They use an OT metric called Gromov-Wasserstein distance capable of comparing two distributions even if they do not lie in the same ground space and apply it to compute barycenter of molecular shapes. Even though this approach has wide applications, it only encodes  the intrinsic structural information in the transportation problem. To the best of our knowledge, the problem of including both structural and feature information in a unified OT formulation remains largely under-addressed.

\paragraph{OT distances that include both features and structures.}
Recent approaches tend to incorporate some structure information as a regularization of the OT problem.
For example in \cite{AlvarezMelis2018Structured} and \cite{courty2017optimal},  authors constrain transport maps to favor some assignments in certain groups. These approaches require a known and simple structure such as class clusters to work but do not generalize well to more general structural information.
In their work \cite{Thorpe2017}, propose an OT distance that combines both a Lagrangian formulation of a signal and its temporal structural information. They define a metric, called Transportation $L^{p}$ distance, that can be seen as a distance over the coupled space of time and feature. They apply it for signal analysis and show that combining both structure and feature tends to better capture the signal information. Yet, for their approach to work, the structure and feature information should lie in the same ambiant space, which is not a valid assumption for more general problems such as similarity between graphs. In \cite{DBLP:conf/aaai/NikolentzosMV17}, authors propose a graph similarity measure for discrete labeled graph with OT. Using the eigenvector decomposition of the adjency matrix, which captures graph connectivities, nodes of a graph are first embedded in a new space, then a ground metric based on the distance in both this embedding and the labels is used to compute a Wasserstein distance serving as a graph similarity measure.

\paragraph{Contributions.}

After defining structured data as probability measures (Section~\ref{sec:back}), we propose a new framework capable of taking into account both structure and feature information into the optimal transport problem. The framework can compare any usual structured machine learning data even if the feature and structure information dwell in spaces of different dimensions, allowing the comparison of undirected labeled graphs. The framework is based on a distance that embeds a  trade-off parameter which allows balancing the importance of the features and the structure. We propose numerical algorithms for computing this distance (Section~\ref{sec:fgw}), and we evaluate it (Section~\ref{sec:expe}) on both synthetic and real-world graph datasets. We also illustrate the notion of graph barycenters in a clustering problem.

\paragraph{Notations.}
The simplex histogram with $n$ bins will be denoted as  $\Sigma_{n}=\{h \in (\mathbb{R}^{*}_{+})^{n},\sum_{i=1}^{n} h_{i}=1,\}$. We note $\otimes$ the tensor-matrix multiplication, \emph{i.e.} for a tensor $L=(L_{i,j,k,l})$, $L\otimes B$ is the matrix $\left(\sum_{k,l} L_{i,j,k,l}B_{k,l}\right)_{i,j}$. $\big\langle . \big\rangle$ is the matrix scalar product associated with the Frobenius norm. For $x \in \Omega$, $\delta_{x}$ denotes the Dirac measure in $x$.

\section{Structured data as probability measures}
\label{sec:back}

In this paper, we focus on comparing structured data which combine a feature \textbf{and} a structure information.
More formally, we consider undirected labeled graphs as tuples of the form $\mathcal{G}=(\mathcal{V},\mathcal{E},\ell_f,\ell_s)$ where $(\mathcal{V},\mathcal{E})$ are the set of vertices and edges of the graph.
$\ell_f: \mathcal{V} \rightarrow \Omega_f$ is a labelling function which associates each vertex $v_{i} \in \mathcal{V}$ with a feature $a_{i}\stackrel{\text{def}}{=}\ell_f(v_{i})$ in some feature metric space $(\Omega_f,d)$. 
We will denote by \emph{feature information} the set of all the features $(a_{i})_i$ of the graph.
Similarly, $\ell_s: \mathcal{V} \rightarrow \Omega_s$ maps a vertex $v_i$ from the graph to its structure representation $x_{i}\stackrel{\text{def}}{=}\ell_s(v_{i})$ in some structure space $(\Omega_s,C)$ specific to each graph. $C : \Omega_s \times \Omega_s \rightarrow \mathbb{R_{+}}$ is a symmetric application which aims at measuring the similarity between the nodes in the graph. Unlike the feature space however, $\Omega_s$ is implicit and in practice, knowing the similarity measure $C$ will be sufficient. With a slight abuse of notation, $C$ will be used in the following to denote both the structure similarity measure and the matrix that encodes this similarity between pairs of nodes in the graph $(C(i,k) = C(x_i, x_k))_{i,k}$.
Depending on the context, $C$ can either encode the neighborhood information of the nodes, the edge information of the graph or more generally it can model a distance between the nodes such as the shortest path distance or the harmonic distance \cite{NIPS2017_6614}. When $C$ is a metric, such as the shortest-path distance, we naturally endow the structure with the metric space $(\Omega_s,C)$.
We will denote by \emph{structure information} the set of all the structure embeddings $(x_{i})_i$ of the graph.

\begin{figure}[t]
    \centering
        \includegraphics[width=1\columnwidth]{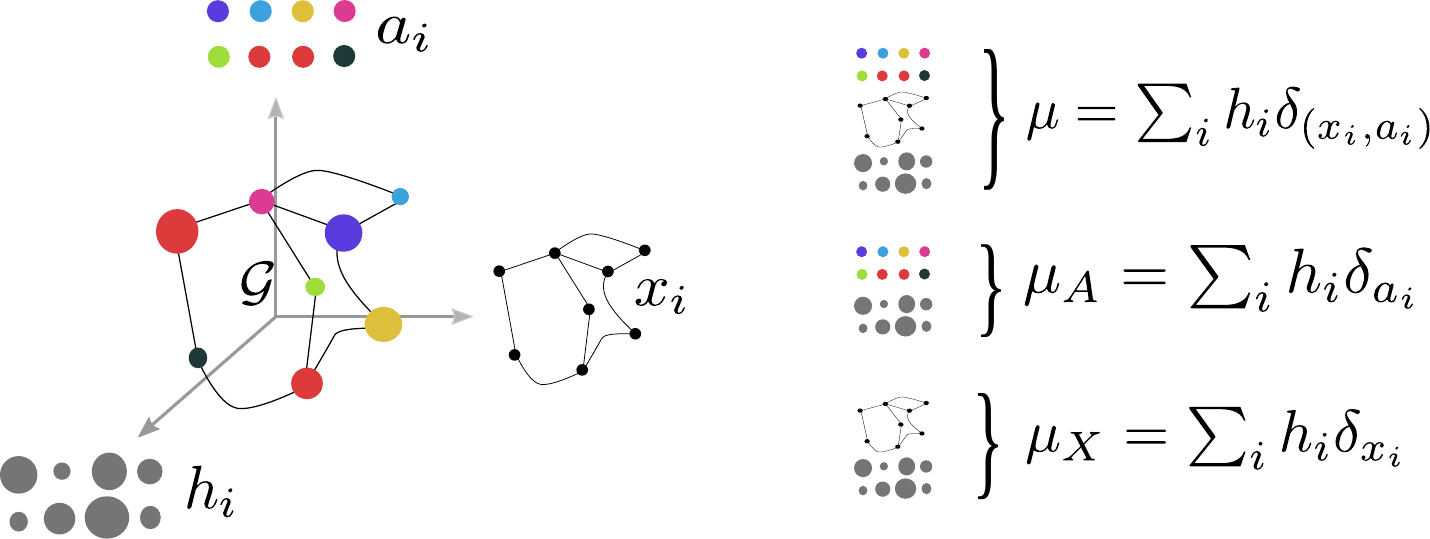}
    \caption{(Left) Labeled graph with $(a_{i})_{i}$ its feature information, $(x_{i})_{i}$ its structure information and histogram $(h_{i})_{i}$ that measures the relative importance of the vertices. (Right) Associated structured data which is entirely described by a fully supported probability measure $\mu$ over the product space of feature and structure, with marginals $\mu_{X}$ and $\mu_{A}$ on the structure and the features respectively.\label{graphex}}
\end{figure}

We propose to enrich the previously described graph with a histogram which serves the purpose of signaling the relative importance of the vertices in the graph. 
To do so, if we assume that the graph has $n$ vertices, we equip those vertices with weights $(h_{i})_{i} \in \Sigma_{n}$. 
Through this procedure, we derive the notion of \emph{structured data} as a tuple $\mathcal{S}=(\mathcal{G},h_{\mathcal{G}})$ where $\mathcal{G}$ is a graph as described previously and $h_{\mathcal{G}}$ is a function that associates a weight to each vertex. This definition allows the graph to be represented by a fully supported probability measure over the product space feature/structure $\mu= \sum_{i=1}^{n} h_{i} \delta_{(x_{i},a_{i})}$ which describes the entire structured data (see Fig. \ref{graphex}). When all the weights are  equal (\textit{i.e.} $h_{i}=\frac{1}{n}$), so all vertices have the same relative importance, the structured data holds the exact same information as its graph. However, weights can be used to encode some \textit{a priori} information. For instance on segmented images, one can construct a graph using the spatial neighborhood of the segmented zones, the features can be taken as the average color in the zone, and the weights as the ratio of image pixels in the zone.

\section{Fused Gromov-Wasserstein approach for structured data}
\label{sec:fgw}
We aim at defining a distance between two graphs $\mathcal{G}_1$ and  $\mathcal{G}_2$, described respectively by their probability measure $\mu= \sum_{i=1}^{n} h_{i} \delta_{(x_{i},a_{i})}$ and $\nu= \sum_{i=1}^{m} g_{j} \delta_{(y_{j},b_{j})}$, where  $h \in \Sigma_{n}$ and $g \in \Sigma_{m}$ are histograms. Without loss of generality we suppose $(x_{i},a_{i})\neq (x_{j},a_{j})$ for $i\neq j$ (same for $y_{j}$ and $b_{j}$).

We introduce $\couplingset(h,g)$ the set of all admissible couplings between $h$ and $g$, \emph{i.e.} the set : 
$$\couplingset(h,g)=\{ \pi \in \mathbb{R}_{+}^{n \times m} \ \textit{s.t.} \ \sum_{i=1}^n\pi_{i,j}=h_j \ , \sum_{j=1}^m\pi_{i,j}=g_i \},$$ 
where $\pi_{i,j}$ represents the amount of mass shifted from the bin $h_{i}$ to $g_{j}$ for a coupling $\pi$. To that extent, the matrix $\pi$ describes a probabilistic matching of the nodes of the two graphs.
$M_{AB}=(d(a_{i},b_{j}))_{i,j}$ is a $n \times m$ matrix standing for the distance between the features.
The structure matrices are denoted $C_{1}$ and $C_{2}$, and $\mu_{X}$ and $\mu_{A}$ (resp. $\nu_{Y}$ and $\nu_{B}$) are representative of the marginals of $\mu$ (resp. $\nu$) \textit{w.r.t.} the structure and feature respectively (see Fig. \ref{graphex}).
We also define the similarity between the structures by measuring the similarity between all pairwise distances within each graph thanks to the 4-dimensional tensor $L(C_{1},C_{2})$:
\begin{equation*}
    L_{i,j,k,l}(C_{1},C_{2})=|C_{1}(i,k)-C_{2}(j,l)| .
\end{equation*}

\subsection{$FGW$ distance}
\begin{figure}
    \centering
    \includegraphics[width=0.9\linewidth]{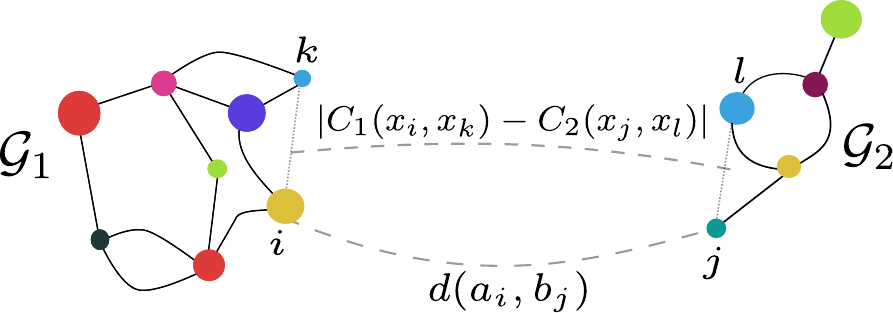}
    \caption{\label{def_fgw}    $FGW$ loss $E_{q}$ for a coupling $\pi$ depends on both a similarity between each feature of each node of each graph $(d(a_{i},b_{j}))_{i,j}$ and between all intra-graph structure similarities $(\left|C_{1}(x_i, x_k)-C_{2}(x_j, x_l)\right|)_{i,j,k,l}$.}
 \end{figure}

We define a novel Optimal Transport discrepancy called the Fused Gromov-Wasserstein distance. It is defined for a trade-off parameter  $\alpha \in [0,1]$ as
\begin{equation}
\label{discretefgw}
\fgwdistance= \underset{\pi \in \couplingset(h,g)}{\min}E_{q}(M_{AB},C_{1},C_{2},\pi)  \\
\end{equation}
where
\begin{equation*}
{\small
\begin{split}
E_{q}(M_{AB},C_{1},C_{2},\pi) = \big\langle (1-\alpha) M_{AB}^{q} + \alpha L(C_{1},C_{2})^{q} \otimes \pi, \pi \big\rangle \\
 =\sum_{i,j,k,l} (1-\alpha) d(a_{i},b_{j})^{q}+\alpha |C_{1}(i,k)-C_{2}(j,l)|^{q} \pi_{i,j}\pi_{k,l}
\end{split}
}
\end{equation*}
The $FGW$ distance looks for the coupling $\pi$ between the vertices of the graph
that minimizes the cost $E_{q}$ which is a linear combinaison of a cost
$d(a_{i},b_{j})$ of transporting one feature $a_{i}$ to a feature $b_{j}$ and a
cost $|C_{1}(i,k)-C_{2}(j,l)|$ of transporting pairs of nodes in each structure
(see Fig. \ref{def_fgw}). As such, the optimal coupling tends to associate pairs of feature and
structure points with similar distances within each structure pair and with
similar features. 
As an important feature of $FGW$, by relying on a sum of (inter- and intra-)vertex-to-vertex distances, it can handle structured data with continuous attributed or discrete labeled nodes (thanks to the definition of $d$) and can also be computed even if the graphs have different number of nodes.

This new distance is called the $FGW$ distance as it acts as a generalization of the Wasserstein \cite{Villani} and Gromov-Wasserstein  \cite{springerlink:10.1007/s10208-011-9093-5,solomon_entropic_gromov2016} distances as stated in the following theorem:
 
\begin{theorem}{Interpolation properties.}
\label{interpolationtheorem}

\noindent As $\alpha$ tends to zero, the $FGW$ distance recovers the Wasserstein distance between the features ${W}_{q}(\mu_{A},\nu_{B})^{q}$
\begin{equation*}
\begin{split}
\lim\limits_{\alpha \rightarrow 0} \fgwdistance&={W}_{q}(\mu_{A},\nu_{B})^{q}=\underset{\pi \in \couplingset(h,g)}{\min} \langle \pi,M_{AB}^{q} \rangle\\
 \end{split}
 \end{equation*}
and as $\alpha$ tends to one, we recover the Gromov-Wasserstein distance ${GW}_{q}(\mu_{X},\nu_{Y})^{q}$ between the structures:
 \begin{equation*}
 \begin{split}
\lim\limits_{\alpha \rightarrow 1}\fgwdistance&={GW}_{q}(\mu_{X},\nu_{Y})^{q} \\
&= \underset{\pi \in \couplingset(h,g)}{\min} \langle L(C_{1},C_{2})^{q} \otimes \pi, \pi \rangle
\end{split}
\end{equation*}

\end{theorem}
Proof of this theorem can be found in the supplementary material.

Similarly to the Wasserstein and Gromov-Wasserstein distances, $FGW$ enjoys metric properties over the space of structured data as stated in the following theorem:
\begin{theorem} ${FGW}$ defines a metric for $q=1$ and a semi-metric for $q >1$.

If $q=1$, and if $C_{1}$, $C_{2}$ are distance matrices then ${FGW}$ defines a metric over the space of structured data quotiented by the measure preserving isometries that are also feature preserving. More precisely, ${FGW}$ satisfies the triangle inequality and is nul \textit{iff} $n=m$ and there exists a bijection $\sigma : \{1,..,n\} \rightarrow \{1,..,n\}$ such that : 
\begin{equation}
\label{weightpreserve}
\forall i \in \{1,..,n\}, \ h_{i}=g_{\sigma(i)}
\end{equation}
\begin{equation}
\label{featurepreserve}
\forall i \in \{1,..,n\}, \ a_{i}=b_{\sigma(i)} 
\end{equation}
\begin{equation}
\label{structurepreserve}
\forall i,k \in  \{1,..,n\}^2,  \ C_{1}(i,k)=C_{2}(\sigma(i),\sigma(k))
\end{equation}
If $q>1$,  the triangle inequality is relaxed by a factor $2^{q-1}$ such that ${FGW}$ defines a semi-metric.
\end{theorem}

All proofs can be found in the supplementary material. The resulting application $\sigma$ preserves the weight of each node (eq. \eqref{weightpreserve}),  the features (eq. \eqref{featurepreserve}) and the and the pairwise structure relation between the nodes (eq. \eqref{structurepreserve}).  For example, comparing two graphs with uniform weights for the vertices and with shortest path structure matrices, the ${FGW}$ distance vanishes \textit{iff} the graphs have the same number of vertices and \textit{iff} there exists a one-to-one mapping between the vertices of the graphs which respects both the shortest paths and the features. More informally, it means that graphs have vertices with the same labels connected by the same edges.

The metric ${FGW}$ is fully unsupervised and can be used in a wide set of applications such as
$k$-nearest-neighbors, distance-substitution kernels, pseudo-Euclidean
embeddings, or representative-set methods. Arguably, such a distance also allows
for a fine interpretation of the similarity (through the optimal mapping $\pi$), contrary
to end-to-end learning machines such as neural networks.

\subsection{Fused Gromov-Wasserstein barycenter}
\label{sec:bary}
OT barycenters have many desirable properties and applications \cite{agueh2011barycenters,peyre2016gromov}, yet no formulation can leverage both structural and feature information in the barycenter computation. 
In this section, we consider the ${FGW}$ distance to define a barycenter of a set of structured data as a Fr\'echet mean. 
We look for the structured data $\mu$ that minimizes the sum of (weighted) ${FGW}$ distances within a given set of structured data $(\mu_{k})_{k}$ associated with structure matrices $(C_{k})_{k}$, features $(B_{k})_{k}$ and base histograms $(h_{k})_{k}$. 
 For simplicity, we assume that the histogram $h$ associated to the barycenter is known and fixed; in other words, we set the number of vertices $N$ and the weight associated to each of them.

In this context, for a fixed $N \in \mathbb{N}$ and $(\lambda_{k})_{k}$ such that $\sum_{k} \lambda_{k}=1$ , we aim to find the set of features $A = (a_i)_i$ and the structure matrix $C$ of the barycenter that minimize the following equation:
\begin{align}
&\underset{\mu}{\text{min}} \sum_{k} \lambda_{k} {FGW}_{q,\alpha}(\mu,\mu_{k}) \label{fgwbarycenter} \\
&=\underset{C \in \mathbb{R}^{N \times N},\ A \in \mathbb{R}^{N \times n},(\pi_{k})_{k}}{\text{min}} \sum_{k} \lambda_{k} E_{q}(M_{AB_{k}},C,C_{k},\pi_{k})\nonumber
\end{align}
Note that this problem is jointly convex \textit{w.r.t.} $C$ and $A$ but not \textit{w.r.t.} $\pi_{k}$. We discuss the proposed algorithm to solve this problem in the next section.  Interestingly enough, one can derive several variants of this problem, where the features or the structure matrices of the barycenter can be fixed. Solving the related simpler optimization problem extends straightforwardly. 
 We give examples of such barycenters both in the experimental section where we solve a graph based $k$-means problem.

\subsection{Optimization and algorithmic solution}

In this section we discuss the numerical optimization problem for computing the ${FGW}$ distance between discrete distributions.

\textbf{Solving the Quadratic Optimization problem.}
Equation \ref{discretefgw} is clearly a quadratic {problem} \emph{w.r.t.} $\pi$.
Note that despite the apparent $\mathcal{O}(m^2n^2)$ complexity of computing the
tensor product, one can simplify the sum to complexity $\mathcal{O}(mn^2+m^2n)$
\cite{peyre2016gromov} when considering $q=2$. In this case,  the
${FGW}$ computation problem can be re-written as finding $\pi^{*}$ such
that:
\begin{equation}
 \pi^{*}=\underset{\pi \in \couplingset(h,g)}{\arg\min}\quad\text{vec}(\pi)^{T} Q(\alpha) \text{vec}(\pi)+ \text{vec}(D(\alpha))^{T}  \text{vec}(\pi)
 \label{l2eq}
 \end{equation}
where $Q=-2 \alpha  C_{2} \otimes_{K} C_{1}$ and $D(\alpha)=(1-\alpha)M_{AB}$.
$\otimes_{K}$  denotes the Kronecker product of two matrices, $\text{vec}$ the
column-stacking operator. With such form, the resulting optimal map can be seen
as a quadratic regularized map from initial Wasserstein \cite{ferradans2014regularized,flamary2014optlaplace}. However, unlike these approaches, we have a quadratic
but provably non convex term.
The gradient  $G$ that arises from Eq.~\eqref{discretefgw} can be expressed with the following partial derivative \emph{w.r.t.} $\pi$:
\begin{equation}
  G=(1-\alpha) M_{AB}^{q}+2\alpha L(C_{1},C_{2})^{q}\otimes \pi
    \label{eq:gradfgw}
\end{equation}
that can be computed with $\mathcal{O}(mn^2+m^2n)$ operations when $q=2$.

\begin{algorithm}[t]
    \caption{\label{alg:cg}
     Conditional Gradient (CG) for $FGW$}
            \begin{algorithmic}[1]
            \STATE $\pi^{(0)}\leftarrow \mu_X\mu_Y^\top$
            \FOR {$i=1,\dots,$}
            \STATE $G\leftarrow$ Gradient from Eq. \eqref{eq:gradfgw} \emph{w.r.t.} $\pi^{(i-1)}$
            \STATE $\tilde\pi^{(i)}\leftarrow $ Solve OT with ground loss $G$
            \STATE $\tau^{(i)}\leftarrow$ Line-search for loss \eqref{discretefgw} with $\tau\in(0,1)$ using Alg. \ref{alg:line}
            \STATE $\pi^{(i)}\leftarrow (1-\tau^{(i)})\pi^{(i-1)}+\tau^{(i)}\tilde\pi^{(i)} $
            \ENDFOR
        \end{algorithmic}
          \end{algorithm}
Solving a large scale QP with a classical solver can be computationally expensive. In \cite{ferradans2014regularized}, authors propose a solver for a graph regularized optimal transport problem whose resulting optimization problem is also a QP. We can then directly use their conditional gradient defined in Alg. \ref{alg:cg} to solve our optimization problem. It only needs at each iteration  to compute the gradient in Eq.~\eqref{eq:gradfgw} and to solve a classical OT problem for instance with a network flow algorithm. The line-search part is a constrained minimization of a second degree polynomial function which is adapted to the non convex loss in Alg. \ref{alg:line}.  While the problem is non convex, conditional gradient is known to converge to a local stationary point~\cite{lacoste2016convergence}.

        \begin{algorithm}[t]
        \caption{\label{alg:line}
         Line-search for CG ($q=2$)}
                    \begin{algorithmic}[1]

        		\STATE $c_{C_{1},C_{2}}$ from Eq. (6) in \cite{peyre2016gromov}
                \STATE $a=-2 \alpha \langle C_{1} \tilde\pi^{(i)} C_{2},\tilde\pi^{(i)} \rangle$ 
                \STATE $b{=}\langle(1{-}\alpha) M_{AB} {+}\alpha c_{C_{1},C_{2}}, \tilde\pi^{(i)} \rangle$
                $\, {-}2\alpha \big( \langle C_{1} \tilde\pi^{(i)} C_{2},\pi^{(i-1)} \rangle + \langle C_{1} \pi^{(i-1)} C_{2}, \tilde\pi^{(i)} \rangle \big)$
                \STATE $c=E_{2}(M_{AB},C_{1},C_{2},\pi^{(i-1)})$
                \IF {$a>0$}                 \STATE $\tau^{(i)} \leftarrow \text{min}(1,\text{max}(0,\frac{-b}{2a}))$
                \ELSE{\STATE $\tau^{(i)} \leftarrow 1$ if $a+b<0$ else $\tau^{(i)} \leftarrow 0$} 
                \ENDIF

            \end{algorithmic}
                  \end{algorithm}

\textbf{Solving the barycenter problem with Block Coordinate Descent (BCD).}
We propose to minimize eq. \eqref{fgwbarycenter} using a BCD algorithm, \textit{i.e.} iteratively minimizing with respect to the couplings $\pi_{k}$, to the metric $C$ and the feature vector $A$.
The minimization of this problem \emph{w.r.t.} $(\pi_{k})_{k}$ is equivalent to compute $S$ independent Fused Gromov-Wasserstein distances as discussed above. We suppose that the feature space is $\Omega_f=(\mathbb{R}^{d},\ell_{2}^2)$ and we consider $q=2$.
Minimization \emph{w.r.t.} $C$ in this case has a closed form (see Prop. 4 in
\cite{peyre2016gromov}) :
\begin{equation*}
    C \leftarrow \frac{1}{hh^{T}} \sum\nolimits_{k} \lambda_{k} \pi_{k}^{T} C_{k} \pi_{k}
\end{equation*}
where $h$ is the histogram of the barycenter as discussed in section \ref{sec:bary}.
Minimization \emph{w.r.t.} $A$ can be computed with (eq. (8) in \cite{pmlr-v32-cuturi14}):

$$A \leftarrow \sum_{k} \lambda_{k} B_{k} \pi_{k}^{T} \text{diag}(\frac{1}{h})$$

\section{Experimental results}
\label{sec:expe}

{We now illustrate the behaviour of ${FGW}$ on synthetic and real
datasets. The algorithms presented in the previous section have been implemented
using the Python Optimal Transport toolbox~\cite{flamary2017pot} and will be
released upon publication.

\subsection{Illustration of ${FGW}$ on trees}

  We construct two trees {as illustrated in Figure \ref{mapstoy}, where the 1D node features are shown with colors (in red, features belong to $[0,1]$ and in blue in $[9,10]$).}
 The structure similarity matrices $C_{1}$ and $C_{2}$ are the shortest path between nodes. Figure \ref{mapstoy} illustrates the behavior of the ${FGW}$ distance when the trade-off parameter $\alpha$ changes. The left part recovers the Wasserstein distance ($\alpha=0$):  red nodes are coupled to red ones and the blue nodes to the blue ones. For a alpha close to $1$ (right), we recover the Gromov-Wasserstein distance: all couples of points are coupled to another couple of points, without taking into account the features. Both approaches fail in discriminating the two trees. Finally, for an intermediate $\alpha$ in $FGW$ (center), the bottom and first level structure is preserved as well as the feature matching (red on red and blue on blue), resulting on a positive distance.

\tikzstyle{vertex}=[circle,fill=black,minimum size=4pt,inner sep=0pt]
\tikzstyle{edge} = [draw]
\tikzstyle{transp} = [draw, dashed]
\tikzstyle{feuille1}=[rectangle,draw,fill=blue,text=blue,minimum size=4pt,inner sep=0pt]
\tikzstyle{feuille2}=[circle,draw,fill=red,text=blue,minimum size=4pt,inner sep=0pt]
\tikzstyle{entour}=[ellipse,draw,text=blue]

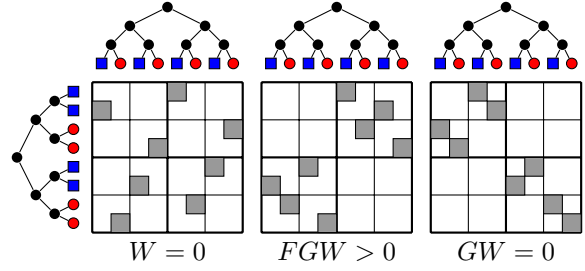
\begin{figure}[t]
\centering
\begin{tikzpicture}[scale=0.25, auto,swap]
    \foreach \pos/\name in {{(0,4)/a}, {(1,2)/c}, {(1,6)/b},
                             {(2,7)/d}, {(2,5)/e}, {(2,3)/f}, {(2,1)/g}}
        		\node[vertex] (\name) at \pos {};
	    \foreach \pos/\name in {
                             {(3,7.5)/h},  {(3,6.5)/i},  {(3,3.5)/l}, {(3,2.5)/m}}
        		\node[feuille1] (\name) at \pos {};
      \foreach \pos/\name in {
                             {(3,5.5)/j},  {(3,4.5)/k},  {(3,0.5)/n}, {(3,1.5)/o}}
       		\node[feuille2] (\name) at \pos {};
                     \foreach \source/ \dest in {b/a, c/a, d/b,e/b, f/c, g/c, d/h, d/i, e/j, e/k, f/l, f/m, g/n, g/o}
        \path[edge] (\source) -- (\dest);

   \draw[black] (4,0) grid[step=2](12,8);
  \draw[black, thick] (4,0) grid[step=4](12,8);

\draw[fill=black!40] (8,7) rectangle(9,8);
\draw[fill=black!40] (4,6) rectangle(5,7);
\draw[fill=black!40] (11,5) rectangle(12,6);
\draw[fill=black!40] (7,4) rectangle(8,5);
\draw[fill=black!40] (10,3) rectangle(11,4);
\draw[fill=black!40] (6,2) rectangle(7,3);
\draw[fill=black!40] (9,1) rectangle(10,2);
\draw[fill=black!40] (5,0) rectangle(6,1);

    \foreach \pos/\name in {{(8,12)/a}, {(6,11)/b}, {(10,11)/c},
                             {(5,10)/d}, {(7,10)/e}, {(9,10)/f}, {(11,10)/g}}
        		\node[vertex] (\name) at \pos {};
	    \foreach \pos/\name in {
                             {(4.5,9)/h},  {(6.5,9)/j},  {(8.5,9)/l}, {(10.5,9)/n}}
        		\node[feuille1] (\name) at \pos {};
      \foreach \pos/\name in {
                             {(5.5,9)/i},  {(7.5,9)/k},  {(9.5,9)/m}, {(11.5,9)/o}}
       		\node[feuille2] (\name) at \pos {};
    \foreach \source/ \dest in {b/a, c/a, d/b,e/b, f/c, g/c, d/h, d/i, e/j, e/k, f/l, f/m, g/n, g/o}
        \path[edge] (\source) -- (\dest);

\draw  (8,-1) node {${W}=0$};

   \draw[black] (12,0) grid[step=2, ,xshift=1cm](20,8);
  \draw[step=4,black, thick, xshift=1cm] (12,0) grid (20,8);
\draw  (17,-1) node {${FGW}>0$};

    \foreach \pos/\name in {{(8+9,12)/a}, {(6+9,11)/b}, {(10+9,11)/c},
                             {(5+9,10)/d}, {(7+9,10)/e}, {(9+9,10)/f}, {(11+9,10)/g}}
        		\node[vertex] (\name) at \pos {};
	    \foreach \pos/\name in {
                             {(4.5+9,9)/h},  {(6.5+9,9)/j},  {(8.5+9,9)/l}, {(10.5+9,9)/n}}
        		\node[feuille1] (\name) at \pos {};
      \foreach \pos/\name in {
                             {(5.5+9,9)/i},  {(7.5+9,9)/k},  {(9.5+9,9)/m}, {(11.5+9,9)/o}}
       		\node[feuille2] (\name) at \pos {};
    \foreach \source/ \dest in {b/a, c/a, d/b,e/b, f/c, g/c, d/h, d/i, e/j, e/k, f/l, f/m, g/n, g/o}
        \path[edge] (\source) -- (\dest);

   \draw[black] (24,0) grid[step=2, ,xshift=-2cm](32,8);
  \draw[step=4,black, thick, xshift=-2cm] (24,0) grid (32,8);

\draw[fill=black!40] (17,7) rectangle(18,8);
\draw[fill=black!40] (19,6) rectangle(20,7);
\draw[fill=black!40] (18,5) rectangle(19,6);
\draw[fill=black!40] (20,4) rectangle(21,5);
\draw[fill=black!40] (15,3) rectangle(16,4);
\draw[fill=black!40] (13,2) rectangle(14,3);
\draw[fill=black!40] (14,1) rectangle(15,2);
\draw[fill=black!40] (16,0) rectangle(17,1);

    \foreach \pos/\name in {{(8+18,12)/a}, {(6+18,11)/b}, {(10+18,11)/c},
                             {(5+18,10)/d}, {(7+18,10)/e}, {(9+18,10)/f}, {(11+18,10)/g}}
        		\node[vertex] (\name) at \pos {};
	    \foreach \pos/\name in {
                             {(4.5+18,9)/h},  {(6.5+18,9)/j},  {(8.5+18,9)/l}, {(10.5+18,9)/n}}
        		\node[feuille1] (\name) at \pos {};
      \foreach \pos/\name in {
                             {(5.5+18,9)/i},  {(7.5+18,9)/k},  {(9.5+18,9)/m}, {(11.5+18,9)/o}}
       		\node[feuille2] (\name) at \pos {};
    \foreach \source/ \dest in {b/a, c/a, d/b,e/b, f/c, g/c, d/h, d/i, e/j, e/k, f/l, f/m, g/n, g/o}
        \path[edge] (\source) -- (\dest);
\draw  (26,-1) node {${GW}=0$};

\draw[fill=black!40] (24,7) rectangle(25,8);
\draw[fill=black!40] (25,6) rectangle(26,7);
\draw[fill=black!40] (22,5) rectangle(23,6);
\draw[fill=black!40] (23,4) rectangle(24,5);
\draw[fill=black!40] (27,3) rectangle(28,4);
\draw[fill=black!40] (26,2) rectangle(27,3);
\draw[fill=black!40] (28,1) rectangle(29,2);
\draw[fill=black!40] (29,0) rectangle(30,1);

\end{tikzpicture}
\vspace{-2mm}

\caption{Example of $FGW$, $GW$ and $W$ on synthetic trees. Dark grey color represents a non null $\pi_{i,j}$ value between two nodes $i$ and $j$. (Left) the $W$ distance between the features with $\alpha=0$, (Middle)  $FGW$  (Right) the $GW$ between the structures $\alpha=1$. \label{mapstoy}}
\end{figure}

\subsection{Graph-structured data classification}

\textbf{Datasets}  {We consider 12 widely used benchmark datasets divided into 3 groups. BZR, COX2 \cite{cox2bzr},  PROTEINS, ENZYMES  \cite{enzymes}, CUNEIFORM \cite{kriege.cuneiform} and SYNTHETIC \cite{NIPS2013_5155} are vector attributed graphs. MUTAG  \cite{doi:10.1021/jm00106a046}, PTC-MR \cite{DBLP:journals/corr/KriegeGW16} and NCI1 \cite{nci1} contain graphs with discrete attributes derived from small molecules.   IMDB-B, IMDB-M \cite{Yanardag15} contain unlabeled graphs derived from social networks. All datas are available in \cite{KKMMN2016}. }

\textbf{Experimental setup} {Regarding the feature distance matrix $M_{AB}$
between node features, when dealing with real valued vector attributed graphs, we consider the $\ell_{2}$
 distance between the labels of the vertices. 
In the case of graphs with discrete attributes, we consider two settings: in the first one, we keep the original labels (denoted as \textsc{raw}); we also consider a Weisfeiler-Lehman labeling (denoted as \textsc{wl}) by concatenating the  labels of the neighbors. A vector of size \textsc{h} is created by repeating this procedure \textsc{h} times ~\cite{wlkernel,DBLP:journals/corr/KriegeGW16}. In both cases, we compute the feature distance matrix by using $d(a_{i},b_{j})=\sum_{k=0}^{H} \delta(\tau(a_{i}^{k}),\tau(b_{j}^{k}))$ where $\delta(x,y)=1$ if $x\neq y$ else $\delta(x,y)=0$ and $\tau(a_{i}^{k})$ denotes the concatenated label at iteration $k$ (for $k=0$ original labels are used). Regarding the structure distances $C$, they are computed by considering a
shortest path distance between the vertices. 
  
  For the classification task, we run a SVM using the indefinite kernel matrix $e^{-\gamma {FGW}}$ which
 is seen as a noisy observation of the true positive semidefinite kernel
 \cite{Luss:2007:SVM:2981562.2981682}. We compare classification accuracies with the following
 state-of-the-art graph kernel methods:}  (SPK) denotes the shortest path kernel
 \cite{enzymes}, (RWK) the random walk kernel \cite{Gartner03ongraph}, (WLK) the
 Weisfeler Lehman kernel \cite{wlkernel}, (GK) the graphlet count kernel
 \cite{Shervashidze09efficientgraphlet}. For real valued vector attributes, we
consider the HOPPER kernel (HOPPERK) \cite{NIPS2013_5155} and the propagation kernel
 (PROPAK) \cite{Neumann2016}. We build upon the GraKel library \cite{2018arXiv180602193S} to construct the kernels and C-SVM to perform the classification. We
 also compare $FGW$ with the PATCHY-SAN framework for  CNN on graphs 
 \cite{pmlr-v48-niepert16}(PSCN) building on our own implementation of the method.

{To provide compare between the methods, most papers about graph classification usually perform a nested cross validation (using 9 folds for training, 1 for testing, and reporting the average accuracy of this experiment repeated 10 times) and report accuracies of the other methods taken from the original papers. However, these comparisons are not fair because of the high variance on most datasets \textit{w.r.t.} the folds chosen for training and testing. This is why, in our experiments, the nested cross validation is performed on the same folds for training and testing for \emph{all} methods. In the result tables \ref{tab:vec},\ref{tab:disc} and \ref{tab:no} we add a (*) when the best score does not yield to a significative improvement (based on a Wilcoxon signed rank test on the test scores) compared to the second best one. Note that, because of their small sizes, we repeat the experiments 50 times for MUTAG and PTC-MR datasets. For all methods using SVM, we cross validate the parameter $C \in \{10^{-7},10^{-6},...,10^{7}\}$. The range of the WL parameter \textsc{h} is $\{0,1...,10\}$, and we also compute this kernel with \textsc{h} fixed at $2,4$. The decay factor $\lambda$ for RWK $\{10^{-6},10^{-5}...,10^{-2}\}$, for the GK kernel we set the graphlet size $\kappa =3$ and cross validate the precision level $\epsilon$ and the confidence $\delta$ as in the original paper \cite{Shervashidze09efficientgraphlet}. The $t_{\text{max}}$ parameter for PROPAK is chosen within $\{1,3,5,8,10,15,20\}$. For PSCN, we choose the normalized betweenness centrality as labeling procedure and cross validate the batch size in $\{10, 15,...,35\}$ and number of epochs in $\{10, 20,..., 100\}$. Finally for ${FGW}$, $\gamma$ is cross validated within $\{2^{-10},2^{-9},...,2^{10}\}$ and $\alpha$ is cross validated \textit{via} a logspace search in $[0,0.5]$ and symmetrically $[0.5,1]$ (15 values are drawn).

\begin{table*}[ht]
    \caption{Average classification accuracy on the graph datasets with vector attributes.\label{tab:vec}}
    \begin{center}
\resizebox{1\linewidth}{!}{
\begin{sc}
    \setlength{\tabcolsep}{4pt}
\begin{tabular}{lllllll}
\toprule
{Vector attributes} &                   BZR &          COX2 &     CUNEIFORM &       ENZYMES &        PROTEIN &      SYNTHETIC \\
\midrule
\midrule
\small{FGW sp}             &  \textbf{85.12$\pm$4.15}* &  77.23$\pm$4.86 &  \textbf{76.67$\pm$7.04} &  71.00$\pm$6.76 &   \textbf{74.55$\pm$2.74} &  \textbf{100.00$\pm$0.00} \\
\midrule
\small{HOPPERK}               &  84.15$\pm$5.26 &  \textbf{79.57$\pm$3.46} &  32.59$\pm$8.73 &  45.33$\pm$4.00 &   71.96$\pm$3.22 &   90.67$\pm$4.67 \\
\small{PROPAK}                 &   79.51$\pm$5.02 &  77.66$\pm$3.95 &  12.59$\pm$6.67 &  \textbf{71.67$\pm$5.63}* &   61.34$\pm$4.38 &   64.67$\pm$6.70 \\\midrule
\small{PSCN k=10}                                 &  80.00$\pm$4.47 &  71.70$\pm$3.57 &  25.19$\pm$7.73 &  26.67$\pm$4.77 &  67.95$\pm$11.28 &  \textbf{100.00$\pm$0.00} \\
\small{PSCN k=5 }                                 &  82.20$\pm$4.23 &  71.91$\pm$3.40 &  24.81$\pm$7.23 &  27.33$\pm$4.16 &   71.79$\pm$3.39 &  \textbf{100.00$\pm$0.00} \\
\bottomrule
\end{tabular}
\end{sc}}
\end{center}
\end{table*}

\begin{table}[t!]
    \caption{Average classification accuracy on the graph datasets with discrete attributes.\label{tab:disc}}
    \vspace{1.5mm}
    \resizebox{1\linewidth}{!}{
\begin{sc}
    \setlength{\tabcolsep}{4pt}
\begin{tabular}{llll}
\toprule
{Discrete attr.} &          MUTAG &          NCI1 &           PTC-MR \\
\midrule
\midrule
FGW raw sp               &  83.26$\pm$10.30 &     72.82$\pm$1.46 &  55.71$\pm$6.74 \\
FGW wl h=2 sp       &   86.42$\pm$7.81 &          85.82$\pm$1.16 &  {63.20$\pm$7.68} \\
FGW wl  h=4 sp       &   \textbf{88.42$\pm$5.67} &          \textbf{86.42$\pm$1.63} & \textbf{65.31$\pm$7.90} \\
\midrule
GK k=3                              &   82.42$\pm$8.40 &    60.78$\pm$2.48 &  56.46$\pm$8.03 \\
RWK             &   79.47$\pm$8.17 &         58.63$\pm$2.44 &  55.09$\pm$7.34 \\
SPK             &   82.95$\pm$8.19 &  74.26$\pm$1.53 &          60.05$\pm$7.39 \\
WLK               &   86.21$\pm$8.48 &          85.77$\pm$1.07 &  62.86$\pm$7.23 \\
WLK h=2                              &   86.21$\pm$8.15 &     81.85$\pm$2.28 &  61.60$\pm$8.14 \\
WLK h=4                              &   83.68$\pm$9.13 &     85.13$\pm$1.61 &
62.17$\pm$7.80 \\\midrule
PSCN k=10                           &  83.47$\pm$10.26 &          70.65$\pm$2.58 &  58.34$\pm$7.71 \\
PSCN k=5                            &  83.05$\pm$10.80 &          69.85$\pm$1.79 &  55.37$\pm$8.28 \\
\bottomrule
\end{tabular}
\quad
\end{sc}}
\end{table}

\begin{table}[t!]
    \caption{Average classification accuracy on the graph datasets with no attributes.\label{tab:no}}
    \begin{center}
    \begin{sc}
    \begin{tabular}{lll}
    \toprule
    {Without attribute} &        IMDB-B &       IMDB-M \\
    \midrule
    \midrule
    GW sp  &  \textbf{63.80$\pm$3.49} &  \textbf{48.00$\pm$3.22} \\
    \midrule
    GK k=3                 &  56.00$\pm$3.61 &  41.13$\pm$4.68 \\
    SPK &  55.80$\pm$2.93 &  38.93$\pm$5.12 \\
    \bottomrule
    \end{tabular}
    \end{sc}
    \end{center}
\end{table}

\textbf{Results and discussion}

\textbf{Vector attributed graphs.} The average accuracies reported in Table \ref{tab:vec} show that FGW is a clear state-of-the-art method and
performs best on 4 out of 6 datasets with performances in the error bars of the
best methods on the other two datasets. Results for CUNEIFORM are significantly below those from the original paper \cite{kriege.cuneiform} which can be explained by the fact that the method in this paper uses a graph convolutional approach specially designed for this dataset and that experiment settings are different. In comparison, the other competitive methods are less consistent as they exhibit some good performances on some datasets only.

\begin{figure}[t!]
    \begin{center}
        \includegraphics[width=0.9\linewidth]{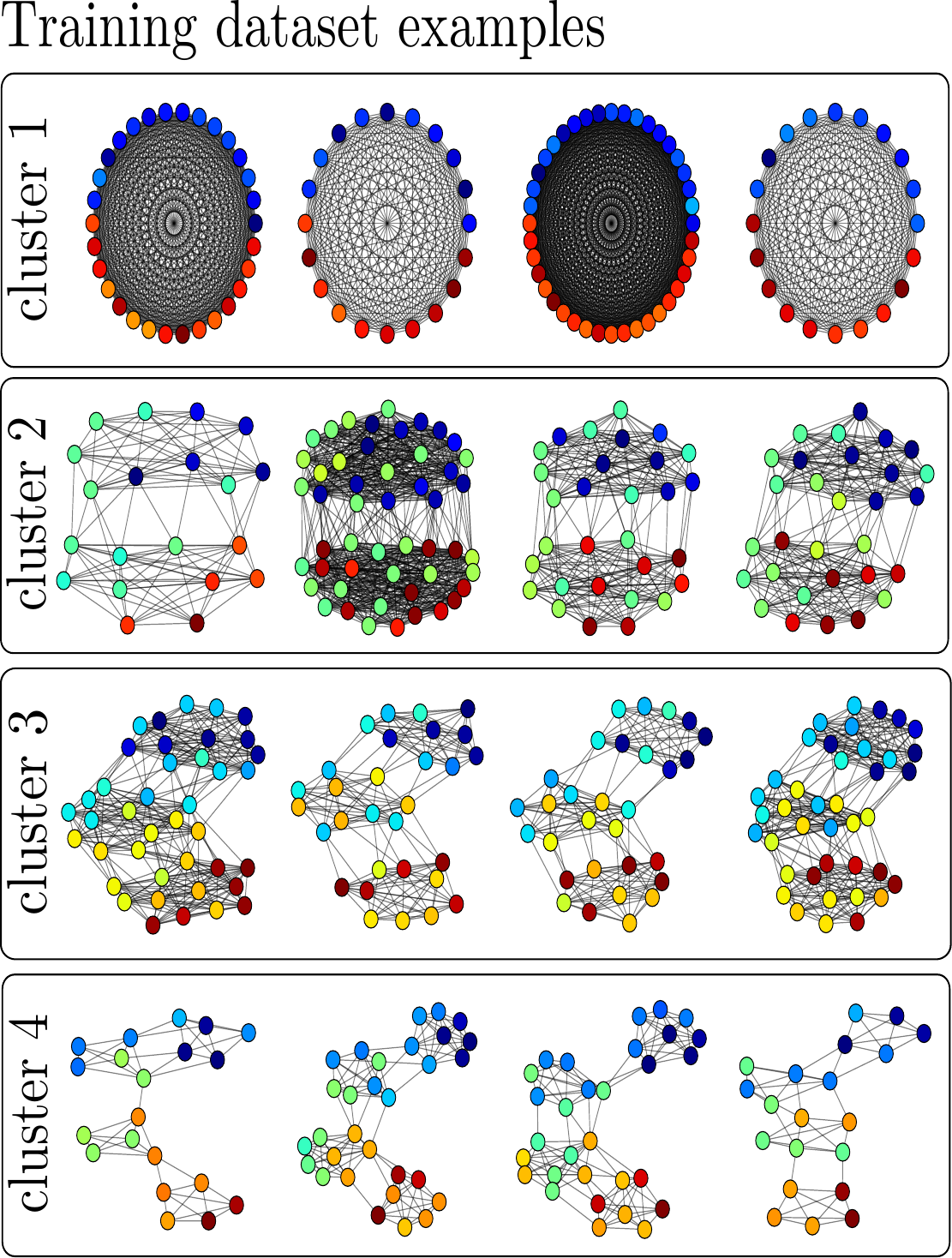} 
    \end{center}
\caption{\label{fig:graphclustr1} {Examples from the clustering dataset, color indicates the labels. } }
\end{figure}
\textbf{Discrete labeled graphs.} We first note in Table \ref{tab:disc} that $FGW$ using WL attributes outperforms all competitive methods, including $FGW$ with raw features. Indeed, the WL attributes allow encoding more finely the neighborood of the vertices by stacking their attributes, whereas FGW with raw features only consider the shortest path distance between vertices, not their sequence of labels. This result calls for using meaningful feature and/or structure matrices in the FGW definition, that can be dataset-dependant, in order to enhance the performances. We also note that $FGW$ with WL attributes outperforms the WL kernel method, highlighting the benefit of an optimal transport-based distance over a kernel-based similarity. Surprisingly results of PSCN are significantly lower than those from the original paper. We believe that it comes from the difference between the folds assignment for training and testing, which suggests that PSCN is difficult to tune.

\textbf{Non-attributed graphs.} The particular case of the GW distance for graph classification is also illustrated on social
datasets, that contain no labels on the vertices. Accuracies reported in Table \ref{tab:no} show that it greatly
outperforms SPK and GK graph kernel methods. This is, to the best of our
knowledge, the first application of GW for social graph classification.

\textbf{Comparison between $FGW$, $W$ and $GW$} During
the validation step, the optimal value of $\alpha$ was consistently selected
inside the $]0,1[$ interval, excluding $0$ and $1$, suggesting that both structure and feature pieces of
information are necessary (details are given in the supplementary material).

\begin{figure}[t!]
    \includegraphics[width=0.9\linewidth]{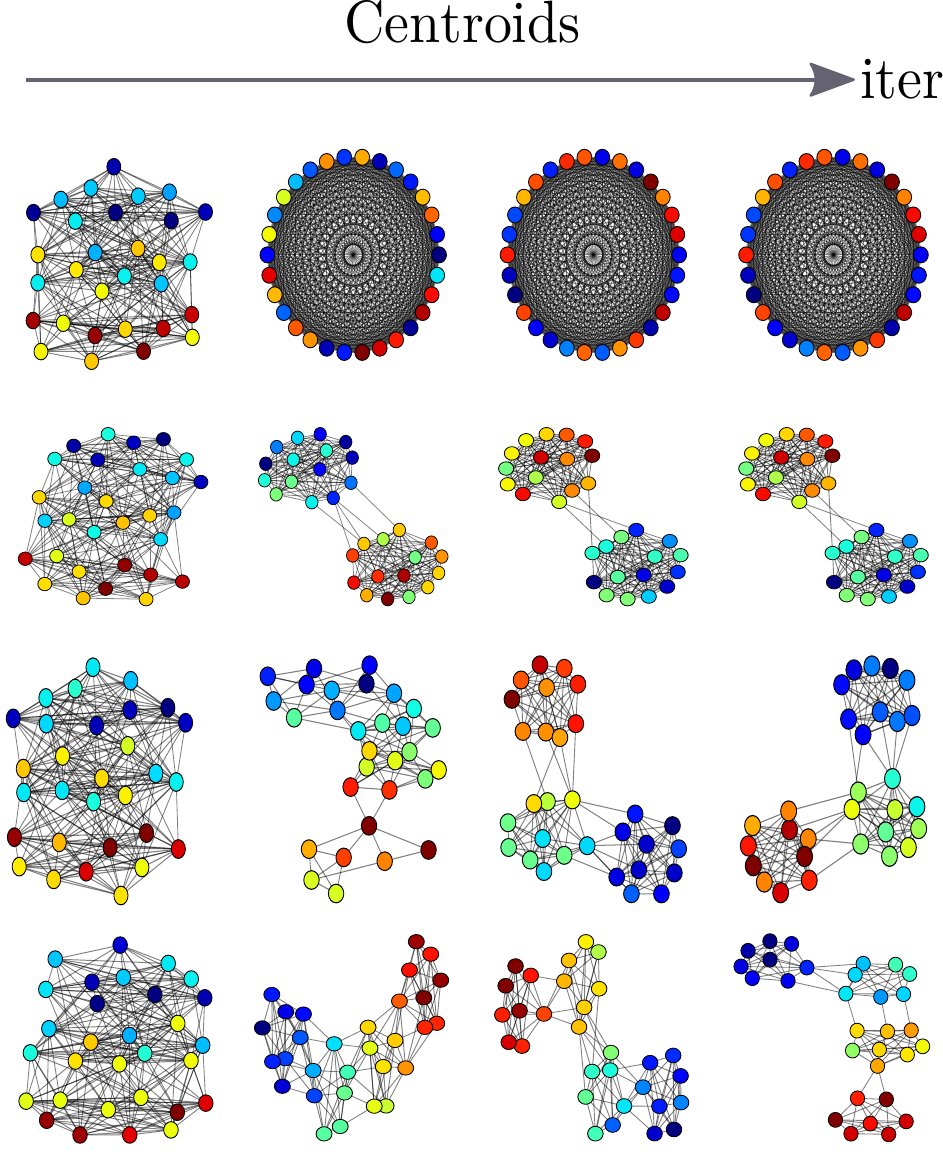} 
    \caption{\label{fig:graphclustr2} {Evolution of the centroids of each cluster in the $k$-means clustering, from (Left) the random initialization (Right) until convergence to the final centroid.} }
    \end{figure}

\subsection{Unsupervised learning: graphs clustering}

{In the last experiment, we evaluate the ability of $FGW$ to perform a
clustering of multiple graphs and to retrieve meaningful barycenters of such
clusters. To do so, we generate a dataset of 4 groups of community graphs. Each
graph follows a simple Stochastic Block Model
\cite{wang1987stochastic,nowicki2001estimation} and the groups are defined 
\textit{w.r.t.} the number of communities inside each graph and the distribution
of their labels. The dataset is composed of 40 graphs (10 graphs per group) and
the number of nodes of each graph is drawn randomly from $\{20,30,...,50\}$
as illustrated in Fig.  \ref{fig:graphclustr1}. We perform a $k$-means clustering using the ${FGW}$ barycenter defined
in eq. \eqref{fgwbarycenter} as the centroid of the groups and the ${FGW}$
distance for the cluster assignment. We fix the number of nodes of each centroid to 30.
We perform a thresholding on the pairwise similarity matrix $C$ of the centroid at
the end in order to obtain an adjacency matrix for visualization purposes. The threshold value is empirically chosen so as to minimize the distance induced by the frobenius norm between the original matrix $C$ and the shortest path matrix obtained from the adjency matrix.   
The evolution of the barycenters along the iterations is reported in  Figure~\ref{fig:graphclustr2}. We can see that these centroids recover
community structures and feature distributions that are representative of their
cluster content. On this example, note that the clustering recovers perfectly
the known groups in the dataset.
To the best of our knowledge, there exists no other method able to perform a clustering of
graphs and to retrieve the average graph in each cluster without having to solve a
pre-image problem.

\section{Discussion and conclusion}

Countless problems in machine learning involve structured data, usually stressed in light of the graph formalism. We consider here labeled graphs enriched by an histogram, which naturally leads to represent structured data as probability measures in the joint space of their features and structures. Widely known for their ability to meaningfully compare probability measures, transportation distances  are generalized in this paper so as to be suited in the context of structured data, motivating the so-called Fused Gromov-Wasserstein distance. We theoretically prove that it defines indeed a distance on structured data, and consequently on graphs of arbitrary sizes. $FGW$ provides a natural framework for {analysis} of labeled graphs as we demonstrate on classification, where it reaches and surpasses most of the time the state-of-the-art performances, and in graph-based $k$-means where we develop a novel approach to represent the clusters centroids  using a {barycentric} formulation of $FGW$. We believe that this metric can have a significant impact on challenging graph signal analysis problems.

{While we considered a unique measure of distance between nodes in the graph structure (shortest path), other choices could be made with respect to the problem at hand, or eventually learned in an end-to-end manner. The same applies to the distance between features.} 
{We also envision a potential use of this distance in deep learning applications where a distance between graph is needed (such as graph auto-encoders). Another line of work will also try to lower the computational complexity of the underlying optimization problem to ensure better scalability to very large graphs.}

\section*{Acknowledgements}

This work benefited from the support from  OATMIL ANR-17-CE23-0012 project of the French National Research Agency (ANR). We gratefully acknowledge the support of NVIDIA Corporation with the donation of the Titan X GPU used for this research.

\bibliographystyle{icml2019}

\section{Supplementary Material}

\subsection{Proofs}

First we recall the notations from the paper :

Let two graphs $\mathcal{G}_1$ and  $\mathcal{G}_2$ described respectively by their probability measure $\mu= \sum_{i=1}^{n} h_{i} \delta_{(x_{i},a_{i})}$ and $\nu= \sum_{i=1}^{m} g_{j} \delta_{(y_{j},b_{j})}$, where  $h \in \Sigma_{n}$ and $g \in \Sigma_{m}$ are histograms with $\Sigma_{n}=\{h \in (\mathbb{R}^{*}_{+})^{n},\sum_{i=1}^{n} h_{i}=1,\}$. 

We introduce $\couplingset(h,g)$ the set of all admissible couplings between $h$ and $g$, \emph{i.e.} the set 
$$\couplingset(h,g)=\{ \pi \in \mathbb{R}_{+}^{n \times m} \ \textit{s.t.} \ \sum_{i=1}^n\pi_{i,j}=h_j \ , \sum_{j=1}^m\pi_{i,j}=g_i \},$$ 
where $\pi_{i,j}$ represents the amount of mass shifted from the bin $h_{i}$ to $g_{j}$ for a coupling $\pi$.

Let $(\Omega_{f},d)$ be a compact measurable space acting as the feature space. We denote the distance between the features as $M_{AB}=(d(a_{i},b_{j}))_{i,j}$, a $n \times m$ matrix. 

The structure matrices are denoted $C_{1}$ and $C_{2}$, and $\mu_{X}$ and $\mu_{A}$ (resp. $\nu_{Y}$ and $\nu_{B}$) the marginals of $\mu$ (resp. $\nu$) \textit{w.r.t.} the structure and feature respectively.
We also define the similarity between the structures by measuring the similarity between all pairwise distances within each graph thanks to the 4-dimensional tensor $L(C_{1},C_{2})$:
\begin{equation*}
    L_{i,j,k,l}(C_{1},C_{2})=|C_{1}(i,k)-C_{2}(j,l)| .
\end{equation*}

We also consider the following notations :

\begin{equation}
\label{lossgromov}
J_{q}(C_1,C_2,\pi)=\sum_{i,j,k,l} L_{i,j,k,l}(C_{1},C_{2})^{q} \pi_{i,j} \, \pi_{k,l}
\end{equation}

\begin{equation}
\label{losswass}
H_{q}(M_{AB},\pi)=\sum_{i,j} d(a_{i},b_{j})^{q} \pi_{i,j}
\end{equation}

\begin{equation}
\label{lossfgw}
{\small
\begin{split}
E_{q}(M_{AB},C_{1},C_{2},\pi) = \big\langle (1-\alpha) M_{AB}^{q} + \alpha L(C_{1},C_{2})^{q} \otimes \pi, \pi \big\rangle \\
 =\sum_{i,j,k,l} (1-\alpha) d(a_{i},b_{j})^{q}+\alpha L_{i,j,k,l}(C_{1},C_{2})^{q} \pi_{i,j}\pi_{k,l}
\end{split}
}
\end{equation}

Respectively $J_{q}$, $H_{q}$ and $E_{q}$ designate the Gromov-Wasserstein ($GW$) loss, the Wasserstein ($W$) loss and the $FGW$ loss so that :

\begin{equation}
\label{discretefgw}
\fgwdistance= \underset{\pi \in \couplingset(h,g)}{\min}E_{q}(M_{AB},C_{1},C_{2},\pi)  \\
\end{equation}

\begin{equation}
\label{discretewass}
W_{q}(\mu_{A},\nu_{B})^{q}=\underset{\pi \in \couplingset(h,g)}{\min} H_{q}(M_{AB},\pi)
\end{equation}

\begin{equation}
\label{discretegromov}
GW_{q}(\mu_{X},\nu_{Y})^{q} = \underset{\pi \in \couplingset(h,g)}{\min} J_{q}(C_1,C_2,\pi)
\end{equation}

Please note that the minimum exists since we minimize a continuous function over a compact subset of $\mathbb{R}^{n\times m}$ and hence the $FGW$ distance is well defined.


\subsubsection{Bounds}

We first introduce the following lemma:

\begin{lemma}
$\fgwdistance$ is lower-bounded by the straight-forward interpolation between $W_{q}(\mu_{A},\nu_{B})^{q}$ and $GW_{q}(\mu_{X},\nu_{Y})^{q}$:

\begin{equation}
\label{ineqwassfgw}
\fgwdistance \geq (1-\alpha)W_{q}(\mu_{A},\nu_{B})^{q} + \alpha GW_{q}(\mu_{X},\nu_{Y})^{q}
\end{equation}
	
\end{lemma}

\begin{proof}
	Let $\pi^\alpha$ be the coupling that minimizes $E_{q}(M_{AB},C_{1},C_{2},\cdot)$.
	Then we have:
	\begin{eqnarray}
		\fgwdistance & = & E_{q}(M_{AB},C_{1},C_{2},\pi^\alpha) \nonumber \\
		& = & (1 - \alpha) H_{q}(M_{AB},\pi^\alpha) + \alpha J_{q}(C_1,C_2,\pi^\alpha) \nonumber
	\end{eqnarray}
	But also:	
	\begin{eqnarray}
		W_{q}(\mu_{A},\nu_{B})^{q} & \leq & H_{q}(M_{AB},\pi^\alpha) \nonumber \\
		GW_{q}(\mu_{X},\nu_{Y})^{q} & \leq & J_{q}(C_1,C_2,\pi^\alpha) \nonumber
	\end{eqnarray}
	The provided inequality is then derived.
\end{proof}

We also have two other straight-forward lower bounds for $FGW$:

\begin{equation}
\label{firstineq}
\fgwdistance \geq (1-\alpha) \ W_{q}(\mu_{A},\nu_{B})^{q} 
\end{equation}

\begin{equation}
\label{secondineq}
\fgwdistance \geq \alpha \ GW_{q}(\mu_{X},\nu_{Y})^{q}
\end{equation}

\subsubsection{Interpolation properties}

We now claim the following theorem:

\begin{theorem}{Interpolation properties.}
\label{interpolationtheorem}

\noindent As $\alpha$ tends to zero, the FGW distance recovers ${W}_{q}(\mu_{A},\nu_{B})^{q}$ between the features, and as $\alpha$ tends to one, we recover ${GW}_{q}(\mu_{X},\nu_{Y})^{q}$ between the structures:
\begin{equation*}
\begin{split}
\lim\limits_{\alpha \rightarrow 0} \fgwdistance&=W_{q}(\mu_{A},\nu_{B})^{q} \\
\lim\limits_{\alpha \rightarrow 1}\fgwdistance&=GW_{q}(\mu_{X},\nu_{Y})^{q}
\end{split}
\end{equation*}

\end{theorem}

\begin{proof}
Let $\pi^{W} \in \couplingset(h,g)$ be the optimal coupling for the Wasserstein distance $W_{q}(\mu_{A},\nu_{B})$ between $\mu_{A}$ and $\nu_{B}$ and let $\pi^{\alpha} \in \couplingset(h,g)$ be the optimal coupling for the $FGW$ distance $\fgwdistance$. We consider :
\begin{equation*}
\begin{split}
&\fgwdistance-(1-\alpha)W_{q}(\mu_{A},\nu_{B})^{q} \\
&= E_{q}(M_{AB},C_{1},C_{2},\pi^{\alpha})- (1-\alpha) H_{q}(M_{AB},\pi^{W}) \\
&\stackrel{*}{\leq} E_{q}(M_{AB},C_{1},C_{2},\pi^{W})- (1-\alpha) H_{q}(M_{AB},\pi^{W}) \\
&=\sum_{i,j,k,l} \alpha |C_{1}(i,k)-C_{2}(j,l)|^{q} \pi^{W}_{i,j}\pi^{W}_{k,l} \\
&= \alpha J_{q}(C_1,C_2,\pi^{W}) \\
\end{split}
\end{equation*}

In (*) we used the suboptimality of the coupling $\pi^{W}$ \textit{w.r.t} the $FGW$ distance. 
In this way we have proven :

\begin{equation}
\label{ineqwassfgw}
\fgwdistance \leq (1-\alpha)W_{q}(\mu_{A},\nu_{B})^{q} + \alpha J_{q}(C_1,C_2,\pi^{W})
\end{equation}

Now let $\pi^{GW} \in \couplingset(h,g)$ the optimal coupling for the Gromov-Wasserstein distance $GW_{q}(\mu_{X},\nu_{Y})$ between $\mu_{X}$ and $\nu_{Y}$. Then :
\begin{equation*}
\begin{split}
&\fgwdistance-\alpha GW_{q}(\mu_{X},\nu_{Y})^{q} \\
&= E_{q}(M_{AB},C_{1},C_{2},\pi^{\alpha})- \alpha J_{q}(C_1,C_2,\pi^{GW}) \\
&\stackrel{*}{\leq} E_{q}(M_{AB},C_{1},C_{2},\pi^{GW})- \alpha J_{q}(C_1,C_2,\pi^{GW}) \\
&=(1-\alpha) \sum_{i,j,k,l} (1-\alpha) d(a_{i},b_{j})^{q}\pi^{GW}_{i,j} \\
&= (1-\alpha) H_{q}(M_{AB},\pi^{GW}) \\
\end{split}
\end{equation*}

where in (*) we used the suboptimality of the coupling $\pi^{GW}$ \textit{w.r.t} the $FGW$ distance so that : 

\begin{equation}
\label{ineqgromovfgw}
\fgwdistance \leq \alpha GW_{q}(\mu_{X},\nu_{Y})^{q} + (1-\alpha) H_{q}(M_{AB},\pi^{GW})
\end{equation}

As $\alpha$ goes to zero Eq.~\eqref{ineqwassfgw} and Eq.~\eqref{firstineq} give $\lim\limits_{\alpha \rightarrow 0} \fgwdistance=W_{q}(\mu_{A},\nu_{B})^{q}$ and as $\alpha$ goes to one Eq.~\eqref{ineqgromovfgw} and Eq.~\eqref{secondineq} give $\lim\limits_{\alpha \rightarrow 1}\fgwdistance =GW_{q}(\mu_{X},\nu_{Y})^{q}$

\end{proof}


\subsubsection{$FGW$ is a distance}

For the following proofs we suppose that $C_{1}$ and $C_{2}$ are distance matrices, $n\geq m$ and $\alpha \in ]0,...,1[$. We claim the following theorem :

\begin{theorem} ${FGW}$ defines a metric for $q=1$ and a semi-metric for $q >1$.

${FGW}$ defines a metric over the space of structured data quotiented by the measure preserving isometries that are also feature preserving. More precisely, ${FGW}$ satisfies the triangle inequality and is nul \textit{iff} $n=m$ and there exists a bijection $\sigma : \{1,..,n\} \rightarrow \{1,..,n\}$ such that :
\begin{equation}
\label{weightpreserve}
\forall i \in \{1,..,n\}, \  h_{i}=g_{\sigma(i)}
\end{equation}\vspace{-4mm}
\begin{equation}
\label{featurepreserve}
\forall i \in \{1,..,n\}, \ a_{i}=b_{\sigma(i)}
\end{equation}\vspace{-6mm}
\begin{equation}
\label{structurepreserve}
\forall i,k \in  \{1,..,n\}^2,  \ C_{1}(i,k)=C_{2}(\sigma(i),\sigma(k))
\end{equation}
If $q>1$,  the triangle inequality is relaxed by a factor $2^{q-1}$ such that ${FGW}$ defines a semi-metric
\end{theorem}

We first prove the equality relation for any $q \geq 1$ and we discuss the triangle inequality in the next section.


\paragraph{Equality relation} 

\begin{theorem} 
For all $q \geq1$, $\fgwdistance=0$ \textit{iff} there exists an application $\sigma : \{1,..,n\} \rightarrow \{1,..,m\}$ which verifies (\ref{weightpreserve}), (\ref{featurepreserve}) and (\ref{structurepreserve})
\end{theorem}

\begin{proof} First, let us suppose that $n=m$ and that such a bijection exists. Then if we consider the transport map $\pi^{*}$ associated with $i \rightarrow i$ and $j\rightarrow \sigma(i)$ \textit{i.e} the map $\pi^{*}=(I_{d} \times \sigma)$ with $I_{d}$ the identity map. 

By eq (\ref{weightpreserve}), $\pi^{*} \in \couplingset(h,g)$ and clearly using (\ref{featurepreserve}) and (\ref{structurepreserve}):

\begin{equation}
\begin{split}
E_{q}(C_{1},C_{2},\pi^{*})&=(1-\alpha) \sum_{i,k} d(a_{i},b_{\sigma(i)})^{q} h_{i} g_{\sigma(i)} h_{k} g_{\sigma(k)}\\
&+\alpha  \sum_{i,k} |C_{1}(i,k)-C_{2}(\sigma(i),\sigma(k))|^{q} h_{i} g_{\sigma(i)} h_{k} g_{\sigma(k)} \\
&=0
\end{split}
\end{equation}

We can conclude that $\fgwdistance=0$.

Conversely, suppose that $\fgwdistance=0$ and $q\geq1$. We define : 

\begin{equation}
\label{hatC1}
\forall i,k \in \{1,...,n\}^{2} , \hat{C_{1}}(i,k)=\frac{1}{2} C_{1}(i,k) + \frac{1}{2} d(a_{i},a_{k})
\end{equation}

\begin{equation}
\label{hatC2}
\forall j,l \in \{1,...,m\}^{2} , \hat{C_{2}}(j,l)=\frac{1}{2} C_{2}(j,l) + \frac{1}{2} d(b_{j},b_{l})
\end{equation}

To prove the existence of a bijection $\sigma$ satisfying the theorem properties we will prove that the Gromov-Wasserstein distance ${GW}_{q}(\hat{C_{1}},\hat{C_{2}},\mu,\nu)$ vanishes.

Let $\pi \in \couplingset(h,g)$ be any admissible transportation plan. Then for $n \geq 1$, :

\begin{equation*}
\begin{split}
&J_{n}(\hat{C_1},\hat{C_2},\pi) =\sum_{i,j,k,l} L(\hat{C_{1}}(i,k),\hat{C_{2}}(j,l))^{n} \pi_{i,j} \pi_{k,l} \\
&=\sum_{i,j,k,l} \left|\frac{1}{2} (C_{1}(i,k)- C_{2}(j,l)) + \frac{1}{2} (d(a_{i},a_{k})-d(b_{j},b_{l})) \right|^{n}  \pi_{i,j} \pi_{k,l} \\
&\stackrel{*}{\leq} \sum_{i,j,k,l} \frac{1}{2}  \left|C_{1}(i,k)- C_{2}(j,l)\right|^{n}  \pi_{i,j} \pi_{k,l} \\
&+ \sum_{i,j,k,l} \frac{1}{2}  \left| d(a_{i},a_{k})-d(b_{j},b_{l}) \right|^{n} \pi_{i,j} \pi_{k,l}
\end{split}
\end{equation*}

In (*) we used the convexity of $t \rightarrow t^{n}$ and Jensen inequality. We denote the first term $(I)$ and $(II)$ the second term. Combining triangle inequalities $d(a_{i},a_{k})\leq d(a_{i},b_{j}) + d(b_{j},a_{k}) $ and $d(b_{j},a_{k})\leq d(b_{j},b_{l}) + d(b_{l},a_{k}) $ we have :

\begin{equation}
\label{triangle}
d(a_{i},a_{k})\leq d(a_{i},b_{j})+ d(a_{k},b_{l})+d(b_{j},b_{l})
\end{equation}

We split (II) in two parts $S_{1}=\{i,j,k,l \ ;  d(a_{i},a_{k})-d(b_{j},b_{l}) \geq 0 \}$ and $S_{2}=\{i,j,k,l \ ;  d(a_{i},a_{k})-d(b_{j},b_{l}) \leq 0 \}$ such that

\begin{equation*}
\begin{split}
 (II) &= \sum_{i,j,k,l \in S_{1}} (d(a_{i},a_{k})-d(b_{j},b_{l}))^{n} \pi_{i,j} \pi_{k,l} \\
&+ \sum_{i,j,k,l \in S_{2}} (d(b_{j},b_{l}))-d(a_{i},a_{k}))^{n} \pi_{i,j} \pi_{k,l}
\end{split}
\end{equation*}
In the same way as Eq.~\eqref{triangle} we have :

\begin{equation}
\label{triangle2}
d(b_{j},b_{l})\leq d(a_{i},a_{k})+d(a_{i},b_{j})+ d(a_{k},b_{l})
\end{equation}

So Eq.~\eqref{triangle} and \eqref{triangle2} give :

\begin{equation}
\begin{split}
(II) & \leq \sum_{i,j,k,l} \frac{1}{2} | d(a_{i},b_{j}) + d(a_{k},b_{l}) |^{n} \pi_{i,j},\pi_{k,l} \\
&\stackrel{def}{=}M_{n}(\pi)
\end{split}
\end{equation}

Finally we have shown that  :

{\small
\begin{equation}
\label{inegalitestyle}
\forall \pi \in \couplingset(h,g), \forall n \geq 1, \ J_{n}(\hat{C_1},\hat{C_2},\pi) \leq \frac{1}{2} J_{n}(C_1,C_2,\pi) + M_{n}(\pi)
\end{equation}
}

Now let $\pi^{*}$ be the optimal coupling for $\fgwdistance$. If $\fgwdistance=0$ then since $E_{q}(C_{1},C_{2},\pi^{*}) \geq \alpha J_{q}(C_1,C_2,\pi^{*}) $ and $E_{q}(C_{1},C_{2},\pi^{*}) \geq (1-\alpha) H_{q}(M_{AB},\pi^{*}) $, we have:

 \begin{equation}
 \label{Jzero}
 J_{q}(C_1,C_2,\pi^{*})=0
 \end{equation}

and

$$H_{q}(M_{AB},\pi^{*})=0$$

Then $\sum_{i,j} d(a_{i},b_{j})^{q} \pi^{*}_{i,j}=0$. Since all terms are positive we can conclude that $\forall m \in \mathbb{N}^{*}, \sum_{i,j} d(a_{i},b_{j})^{m} \pi^{*}_{i,j}=0$. 

In this way :

{\small
\begin{equation}
\begin{split}
M_{q}(\pi^{*})&=\frac{1}{2} \sum_{h} {q\choose p} \big(\sum_{i,j} d(a_{i},b_{j})^{p} \pi^{*}_{i,j}\big) \big(\sum_{k,l} d(a_{k},b_{l})^{q-p} \pi^{*}_{k,l}\big)\\
&=0
\end{split}
\end{equation}
}

Using equations \eqref{inegalitestyle} and \eqref{Jzero} we have shown : $$J_{q}(\hat{C_1},\hat{C_2},\pi^{*})=0$$

So $\pi^{*}$ is the optimal coupling for ${GW}_{q}(\hat{C_{1}},\hat{C_{2}},\mu,\nu)$ and ${GW}_{q}(\hat{C_{1}},\hat{C_{2}},\mu,\nu)=0$. By virtue to Gromov-Wasserstein properties (see \cite{springerlink:10.1007/s10208-011-9093-5}), there exists an isomorphism between the metric spaces associated with $\mu$ and $\nu$. In the discrete case this results in the existence of a function $\sigma : \{1,..,m\} \rightarrow \{1,..,n\}$ which is a weight preserving isometry and thus bijective. In this way, we have $m=n$ and $\sigma$ verifiying Eq~\eqref{weightpreserve}. The isometry property leads also to :

\begin{equation}
\label{hat}
\forall i,k \in  \{1,..,n\}^{2}, \ \hat{C_{1}}(i,k)=\hat{C_{2}}(\sigma(i),\sigma(k))
\end{equation}

Moreover, since $\pi^{*}$ is the optimal coupling for ${GW}_{q}(\hat{C_{1}},\hat{C_{2}},\mu,\nu)$ leading to a zero cost,  then $\pi^{*}$ is supported by $\sigma$, in particular $\pi^{*}=(I_{d} \times \sigma)$ 

So  $H_{q}(M_{AB},\pi^{*})=\sum_{i} d(a_{i},b_{\sigma(i)})^{q} h_{i} g_{\sigma(i)}$. Since $H_{q}(M_{AB},\pi^{*})=0$ and all the weights are strictly positive we can conclude that $a_{i}=b_{\sigma(i)}$.

In this way, $d(a_{i},a_{k})=d(b_{\sigma(i)},b_{\sigma(k)})$, so using the equality (\ref{hat}) and the definition of $\hat{C_1}$ and $\hat{C_2}$ in \eqref{hatC1} and \eqref{hatC2} we can conclude that :
$$\forall i,k \in  \{1,..,n\} \times  \{1,..,n\},  \ C_{1}(i,k)=C_{2}(\sigma(i),\sigma(k))$$
which concludes the proof.

\end{proof}


\paragraph{Triangle Inequality}

\begin{theorem} 
For all $q=1$, $FGW$ verifies the triangle inequality.
\end{theorem}

\begin{proof}
To prove the triangle inequality of $FGW$ distance for arbitrary measures we will use the gluing lemma (see \cite{Villani}) which stresses the existence of couplings with a prescribed structure. Let $h, \ g, \ f \in \Sigma_{n} \times \Sigma_{m} \times \Sigma_{k}$. Let also $\mu=\sum_{i=1}^{n} h_{i} \delta_{a_{i},x_{i}}$, $\nu=\sum_{j=1}^{m} g_{j} \delta_{b_{j},y_{j}}$ and $\gamma=\sum_{p=1}^{k} f_{p} \delta_{c_{p},z_{p}}$ be three structured data as described in the paper. We note $C_{1}(i,k)$ the distance between vertices $x_{i}$ and $x_{k}$, $C_{2}(i,k)$ the distance between vertices $y_{i}$ and $y_{k}$ and $C_{3}(i,k)$ the distance between vertices $z_{i}$ and $z_{k}$.

Let $P$ and $Q$ be two optimal solutions of the $FGW$ transportation problem between $\mu$ and $\nu$ and $\nu$ and $\gamma$ respectively. 

We define :

$$S=P\text{diag}(\frac{1}{g})Q$$ (note that $S$ is well defined since $g_j \neq 0$ for all $j$). Then by definition $S \in \couplingset(h,f)$ because :

$S 1_{m}=P\text{diag}(\frac{1}{g})Q 1_{m} = P\text{diag}(\frac{g}{g}) =P1_{m}=h$ (same reasoning for $f$).

We first prove the triangle inequality for the case $q=1$.

By suboptimality of $S$ :

{\small
\begin{equation*}
\begin{split}
&FGW_{1,\alpha}(C_{1},C_{3},\mu,\gamma) \\
&\leq \sum_{i,j,k,l} (1-\alpha) d(a_{i},c_{j})+\alpha L(C_{1}(i,k),C_{3}(j,l)) S_{i,j} S_{k,l} \\
&= \sum_{i,j,k,l} ((1-\alpha) d(a_{i},c_{j})+\alpha L(C_{1}(i,k),C_{3}(j,l)) \\
&\times  \big(\sum_{e} \frac{P_{i,e} Q_{e,j}}{g_{e}} \big) \big(\sum_{o} \frac{P_{k,o} Q_{o,l}}{g_{o}} \big) \\
&= \sum_{i,j,k,l} \big((1-\alpha) d(a_{i},c_{j})+\alpha |C_{1}(i,k)-C_{3}(j,l)|\big) \\
&\times  \big(\sum_{e} \frac{P_{i,e} Q_{e,j}}{g_{e}} \big) \big(\sum_{o} \frac{P_{k,o} Q_{o,l}}{g_{o}} \big) \\
&\stackrel{*}{\leq}\sum_{i,j,k,l,e,o} \big((1-\alpha) (d(a_{i},b_{e})+d(b_{e},c_{j}))\\
&+\alpha |C_{1}(i,k)-C_{2}(e,o)+C_{2}(e,o)-C_{3}(j,l)|\big) \\
&\times  \big(\frac{P_{i,e} Q_{e,j}}{g_{e}} \big) \big( \frac{P_{k,o} Q_{o,l}}{g_{o}} \big) \\
&\stackrel{**}{\leq} \sum_{i,j,k,l,e,o} \big((1-\alpha)d(a_{i},b_{e})+\alpha L(C_{1}(i,k),C_{2}(e,o))\big) \\
& \times \frac{P_{i,e} Q_{e,j}}{g_{e}} \frac{P_{k,o} Q_{o,l}}{g_{o}} \\
&+ \sum_{i,j,k,l,e,o} \big((1-\alpha) d(b_{e},c_{j})+ \alpha L(C_{2}(e,o),C_{3}(j,l))\big) \\
& \times \frac{P_{i,e} Q_{e,j}}{g_{e}} \frac{P_{k,o} Q_{o,l}}{g_{o}} \\
\end{split}
\end{equation*}
}

where in (*) we use the triangle inequality of $d$ and in (**) the triangle inequality of $|.|$

Moreover we have :

\begin{equation*}
\label{conservation}
\sum_{j} \frac{Q_{e,j}}{g_{e}}=1, \ \sum_{l} \frac{Q_{o,l}}{g_{o}}=1, \ \sum_{i} \frac{P_{i,e}}{g_{e}}=1, \ \sum_{k} \frac{P_{k,o}}{g_{o}}=1
\end{equation*}

So,
{\small
\begin{equation*}
\begin{split}
&FGW_{1,\alpha}(C_{1},C_{3},\mu,\gamma) \\ 
&\leq \sum_{i,k,e,o} \big((1-\alpha)d(a_{i},b_{e}) + \alpha L(C_{1}(i,k),C_{2}(e,o))\big) P_{i,e} P_{k,o} \\ 
&+\sum_{l,j,e,o} \big((1-\alpha)d(b_{e},c_{j}) + \alpha L(C_{2}(e,o),C_{3}(j,l))\big) Q_{e,j} Q_{o,l} 
\end{split}
\end{equation*}
}

Since $P$ and $Q$ are the optimal plans we have :

{\small
\begin{equation*}
\begin{split}
FGW_{1,\alpha}(C_{1},C_{3},\mu,\gamma) &\leq FGW_{1,\alpha}(C_{1},C_{2},\mu,\nu) \\
&+ FGW_{1,\alpha}(C_{2},C_{3},\nu,\gamma)
\end{split}
\end{equation*}
}

which prove the triangle inequality for $q=1$. 

\end{proof}

\begin{theorem} 
For all $q>1$, $FGW$ verifies the relaxed triangle inequality :
{\small
\begin{equation*}
\begin{split}
FGW_{q,\alpha}(C_{1},C_{3},\mu,\gamma) &\leq 2^{q-1}\big(FGW_{q,\alpha}(C_{1},C_{2},\mu,\nu)\\
 &+ FGW_{q,\alpha}(C_{2},C_{3},\nu,\gamma)\big)
\end{split}
\end{equation*}
}
\end{theorem}

\begin{proof}

Let $q > 1$, We have :

\begin{equation}
\label{holder}
\forall x,y \in \mathbb{R}_{+}, \ (x+y)^{q} \leq 2^{q-1} \big(x^{q} + y^{q}\big)
\end{equation}

Indeed,

$(x+y)^{q} = \big((\frac{1}{2^{q-1}})^{\frac{1}{q}} \frac{x}{(\frac{1}{2^{q-1}})^{\frac{1}{q}}} + (\frac{1}{2^{q-1}})^{\frac{1}{q}} \frac{y}{(\frac{1}{2^{q-1}})^{\frac{1}{q}}}\big)^{q} \leq \big[ (\frac{1}{2^{q-1}})^{\frac{1}{q-1}} +(\frac{1}{2^{q-1}})^{\frac{1}{q-1}}\big]^{q-1} \big(\frac{x^{q}}{\frac{1}{2^{q-1}}} + \frac{y^{q}}{\frac{1}{2^{q-1}}}\big) \\
= \frac{x^{q}}{\frac{1}{2^{q-1}}} + \frac{y^{q}}{\frac{1}{2^{q-1}}}$

Last inequality is a consequence of Hölder inequality. Then using same notations :

{\small
\begin{equation*}
\begin{split}
&FGW_{q,\alpha}(C_{1},C_{3},\mu,\gamma) \\
&\leq \sum_{i,j} \sum_{k,l} (1-\alpha) d(a_{i},c_{j})^{q}+\alpha L(C_{1}(i,k),C_{3}(j,l))^{q} S_{i,j} S_{k,l} \\
&= \sum_{i,j} \sum_{k,l} ((1-\alpha) d(a_{i},c_{j})^{q}+\alpha L(C_{1}(i,k),C_{3}(j,l))^{q} \\
&\times  \big(\sum_{e} \frac{P_{i,e} Q_{e,j}}{g_{e}} \big) \big(\sum_{o} \frac{P_{k,o} Q_{o,l}}{g_{o}} \big) \\
&= \sum_{i,j,k,l} \big((1-\alpha) d(a_{i},c_{j})^{q}+\alpha |C_{1}(i,k)-C_{3}(j,l)|^{q}\big) \\
&\times  \big(\sum_{e} \frac{P_{i,e} Q_{e,j}}{g_{e}} \big) \big(\sum_{o} \frac{P_{k,o} Q_{o,l}}{g_{o}} \big) \\
&\stackrel{*}{\leq}\sum_{i,j,k,l,e,o} \big((1-\alpha) (d(a_{i},b_{e})+d(b_{e},c_{j}))^{q}\\
&+\alpha |C_{1}(i,k)-C_{2}(e,o)+C_{2}(e,o)-C_{3}(j,l)|^{q}\big) \\
&\times  \big(\frac{P_{i,e} Q_{e,j}}{g_{e}} \big) \big( \frac{P_{k,o} Q_{o,l}}{g_{o}} \big) \\
&\stackrel{**}{\leq} 2^{q-1}\sum_{i,j,k,l,e,o} \big((1-\alpha)d(a_{i},b_{e})^{q}+\alpha L(C_{1}(i,k),C_{2}(e,o))^{q}\big) \\
& \times \frac{P_{i,e} Q_{e,j}}{g_{e}} \frac{P_{k,o} Q_{o,l}}{g_{o}} \\
&+ 2^{q-1}\sum_{i,j,k,l,e,o} \big((1-\alpha) d(b_{e},c_{j})^{q}+ \alpha L(C_{2}(e,o),C_{3}(j,l))^{q}\big) \\
& \times \frac{P_{i,e} Q_{e,j}}{g_{e}} \frac{P_{k,o} Q_{o,l}}{g_{o}} \\
\end{split}
\end{equation*}
}

where in (*) we use the triangle inequality of $d$ and in (**) the triangle inequality of $|.|$ and \eqref{holder}.

Since $P$ and $Q$ are the optimal plans we have :

{\small
\begin{equation*}
\begin{split}
FGW_{q,\alpha}(C_{1},C_{3},\mu,\gamma) &\leq 2^{q-1}\big(FGW_{q,\alpha}(C_{1},C_{2},\mu,\nu)\\
 &+ FGW_{q,\alpha}(C_{2},C_{3},\nu,\gamma)\big)
\end{split}
\end{equation*}
}

Which prove that $FGW_{q,\alpha}$ defines a semi metric for $q>1$ with coefficient $2^{q-1}$ for the triangle inequality relaxation.

\end{proof}

\subsection{Comparaison with $W$ and $GW$}

\paragraph{Cross validation results}

\begin{table}
\label{discretetable}
    \caption{Percentage of $\alpha$ chosen in $]0,...,1[$ compared to $\{0,1\}$ for discrete labeled graphs}
\begin{small}
    \resizebox{1\columnwidth}{!}{
\begin{sc}
    \setlength{\tabcolsep}{4pt}
\begin{tabular}{llll}
\toprule
{Discrete attr.} &          MUTAG &          NCI1 &           PTC \\
\midrule
\midrule
FGW raw sp               &  100\% &     100\% &  98\% \\
FGW wl h=2 sp       &   100\% &          100\% & 88\% \\
FGW wl  h=4 sp       &   100\% &          100\% &  88\% \\
\midrule
\end{tabular}
\quad
\end{sc}}
\end{small}
\end{table}

During the nested cross validation, we divided the dataset into 10 and use 9 folds for training, where $\alpha$ is chosen within $[0,1]$ \textit{via} a 10-CV cross-validation, 1 fold for testing, with the best value of $\alpha$ (with the best average accuracy on the 10-CV) previously selected.  The experiment is repeated 10 times for each dataset except for MUTAG and PTC where it is repeated 50 times. Table \ref{discretetable} and \ref{continuoustable} report the average number of time $\alpha$ was chose within $]0,...1[$ without $0$ and $1$ corresponding to the Wasserstein and Gromov-Wasserstein distances respectively. Results suggests that both structure and feature pieces of
information are necessary as $\alpha$ is consistently selected inside $]0,...1[$ except for PTC and COX2.

\begin{table}
\label{continuoustable}
    \caption{Percentage of $\alpha$ chosen in $]0,...,1[$ compared to $\{0,1\}$ for vector attributed graphs}
\begin{small}
    \begin{center}
    \resizebox{1\columnwidth}{!}{
\begin{sc}
    \setlength{\tabcolsep}{4pt}
\begin{tabular}{lllllll}
\toprule
{Vector attributes} &                   BZR &          COX2 &     CUNEIFORM &       ENZYMES &        PROTEIN &      SYNTHETIC \\
\midrule
\midrule
\small{FGW sp}             &  100 \% &  90\% &  100\% &  100\% &  100\% &  100\% \\
\midrule
\end{tabular}
\end{sc}}
\end{center}
\end{small}

\end{table}

\begin{table}
    \caption{Average classification accuracy on the graph datasets with discrete attributes.\label{tab:wgw:disc}}
\begin{small}
    \resizebox{1.08\linewidth}{!}{
\begin{sc}
    \setlength{\tabcolsep}{4pt}
\begin{tabular}{llll}
\toprule
{Discrete attr.} &          MUTAG &          NCI1 &           PTC-MR \\
\midrule
\midrule
FGW raw sp               &  83.26$\pm$10.30 &     72.82$\pm$1.46 &  55.71$\pm$6.74 \\
FGW wl h=2 sp       &   86.42$\pm$7.81 &          85.82$\pm$1.16 &  {63.20$\pm$7.68} \\
FGW wl  h=4 sp       &   \textbf{88.42$\pm$5.67} &          \textbf{86.42$\pm$1.63}* & 65.31$\pm$7.90 \\
\midrule
W raw sp       &  79.36$\pm$3.49  &   70.5$\pm$4.63        &  54.79$\pm$5.76 \\
W wl h=2 sp       &  87.78$\pm$8.64  &     85.83$\pm$1.75      & 63.90$\pm$7.66  \\
W wl  h=4 sp       & 87.15$\pm$8.23   &     86.42$\pm$1.64     & \textbf{66.28$\pm$6.95}* \\
\midrule
GW sp &82.73$\pm$9.59 & 73.40$\pm2.80$ &  54.45$\pm$ 6.89\\
\bottomrule
\end{tabular}
\quad
\end{sc}}
\end{small}
\end{table}

\paragraph{Nested CV results} We report in tables \ref{tab:wgw:vec} and \ref{tab:wgw:disc} the average classification accuracies of the nested classification procedure by taking $W$ and $GW$ instead of $FGW$ (\textit{i.e} by taking $\alpha=0,1$). Best result for each dataset is in bold. A (*) is added when best score does not yield to a significative improvement compared to the second best score. The significance is based on a Wilcoxon signed rank test between the best method and the second one.

Results illustrates that $FGW$ encompasses the two cases of $W$ and $GW$, as scores of $FGW$ are usually greater or equal on every dataset than scores of both $W$ and $GW$ and when it is not the case the difference is not statistically significant.

\paragraph{Timings} In this paragraph we provide some timings for the discrete attributed datasets. Table \ref{tab:wgw:timings} displays the average timing for computing $FGW$ between two pair of graphs.

\begin{table}
    \caption{Average timings for the computation of $FGW$ between two pairs of graph \label{tab:wgw:timings}}
\begin{small}
    \resizebox{1.08\linewidth}{!}{
\begin{sc}
    \setlength{\tabcolsep}{4pt}
\begin{tabular}{llll}
\toprule
{Discrete attr.} &          MUTAG &          NCI1 &           PTC-MR \\
\midrule
\midrule
FGW               &  2.5 ms  &  7.3 ms  & 3.7 ms \\
\bottomrule
\end{tabular}
\quad
\end{sc}}
\end{small}
\end{table}

\begin{table*}[ht!]
    \caption{Average classification accuracy on the graph datasets with vector attributes.\label{tab:wgw:vec}}
\begin{small}
    \begin{center}
    \resizebox{.9\textwidth}{!}{
\begin{sc}
    \setlength{\tabcolsep}{4pt}
\begin{tabular}{lllllll}
\toprule
{Vector attributes} &                   BZR &          COX2 &     CUNEIFORM &       ENZYMES &        PROTEIN &      SYNTHETIC \\
\midrule
\midrule
\small{FGW sp}             &  85.12$\pm$4.15 &  \textbf{77.23$\pm$4.86}* &  \textbf{76.67$\pm$7.04} &  71.00$\pm$6.76 &   74.55$\pm$2.74 &  \textbf{100.00$\pm$0.00} \\
\midrule
\small{W}               &  \textbf{85.36$\pm$4.87}* & 77.23$\pm$3.16  &  61.48$\pm$10.23 &\textbf{71.16$\pm$6.32}*  &  \textbf{75.98$\pm$ 1.97}*  &  34.07$\pm$11.33 \\
\midrule
\small{GW sp}                 &  82.92$\pm$6.72  & 77.65$\pm$5.88  & 50.66$\pm$8.91  & 23.66$\pm$3.63  &  71.96$\pm$ 2.40  &  41.66$\pm$4.28  \\
\bottomrule
\end{tabular}
\end{sc}}
\end{center}
\end{small}
\end{table*}

\end{document}